\documentclass[twoside,11pt]{article}

\PassOptionsToPackage{square,sort,comma,numbers}{natbib}

\usepackage{jmlr2e}
\input{Definitions.tex}

\usepackage{times}

\usepackage{algorithmicx}
\usepackage[Algorithm,ruled]{algorithm}
\usepackage{graphicx}
\usepackage{color}
\usepackage{multirow}
\usepackage{multicol}
\usepackage{tikz}
\usepackage{mathtools}
\usepackage[noend]{algpseudocode}
\usepackage{marvosym}
\usepackage{tikzsymbols}
\usetikzlibrary{positioning,automata,decorations.pathreplacing, decorations.markings, shapes.misc, shapes.geometric}
\usepackage{overlock}
\usepackage[T1]{fontenc}
\usepackage{tabularx}

\AtBeginDocument{%
 \abovedisplayskip=4pt plus 2pt minus 2pt
 \abovedisplayshortskip=0pt plus 3pt
 \belowdisplayskip=5pt plus 2pt minus 2pt
 \belowdisplayshortskip=5pt plus 2pt minus 4pt
}

\definecolor{new_red}{rgb}{.8,0,0}
\definecolor{new_blue}{rgb}{0,0,.8}

\newenvironment{proof}{ \par\noindent{\bfseries\upshape Proof\ }}{$\square$}

\jmlrheading{1}{2017}{1-48}{4/00}{10/00}{meila00a}{Holakou Rahmanian, David P. Helmbold, and S.V.N. Vishwanathan}

\ShortHeadings{Online Learning of Combinatorial Objects \\$\,$ via Extended Formulation}{Rahmanian, Helmbold, and Vishwanathan}
\firstpageno{1}

\begin{document}

\title{Online Learning of Combinatorial Objects \\$\,$ via Extended Formulation}

\author{\name Holakou Rahmanian \email holakou@ucsc.edu \\
       \addr Department of Computer Science\\
       University of California Santa Cruz\\
       Santa Cruz, CA 95060, USA
       \AND
       \name David P. Helmbold \email dph@ucsc.edu \\
       \addr Department of Computer Science\\
       University of California Santa Cruz\\
       Santa Cruz, CA 95060, USA
       \AND
       \name S.V.N. Vishwanathan \email vishy@ucsc.edu \\
       \addr Department of Computer Science\\
       University of California Santa Cruz\\
       Santa Cruz, CA 95060, USA
}

\editor{No editors}

\maketitle

\begin{abstract}
The standard techniques for online learning of combinatorial objects
perform multiplicative updates followed by projections into the convex hull of all the objects.
However, this methodology can be expensive 
 if the convex hull contains many facets.
For example, the convex hull of $n$-symbol Huffman trees is known to have exponentially many facets \citep{maurras2010convex}. 
We get around this difficulty by exploiting extended formulations \citep{kaibel2011extended}, which encode the polytope of combinatorial objects in a higher dimensional ``extended'' space with only polynomially many facets. We develop a general framework for converting extended formulations into efficient online algorithms with good relative loss bounds. We present applications of our framework to online learning of Huffman trees and permutations.
The regret bounds of the resulting algorithms are within a factor of $\Ocal(\sqrt{\log(n)})$ of the state-of-the-art specialized algorithms for permutations, 
and depending on the loss regimes, improve on or match the state-of-the-art for Huffman trees. 
Our method is general and can be applied to other combinatorial objects. 
\end{abstract}


\section{Introduction}
\label{sec:intro}

This paper introduces a general methodology for developing efficient and effective algorithms for learning combinatorial structures.
Examples include learning the best permutation of a set of elements for scheduling or assignment problems, 
or learning the best Huffman tree for compressing sequences of symbols.
Online learning algorithms are being successfully applied to an increasing variety of problems, 
so it is important to have good tools and techniques for creating good algorithms that match the particular problem at hand.

\begin{algorithm} 
\floatname{algorithm}{Prediction Game}
\begin{algorithmic}[1]
\State For each trial $t = 1, \ldots, T$
\State \quad The \textbf{learner} predicts (perhaps randomly) with an object $\hvechat_{t-1}$ in class $\Hcal$.
\State \quad  The \textbf{adversary} reveals a loss vector $\ellvec_t \in [0,1]^n$.
\State \quad The \textbf{learner} incurs a (expected) linear loss $\EE[\hvechat_{t-1} \cdot \ellvec_t]$.
\end{algorithmic}
\caption{Prediction game for the combinatorial class $\Hcal \subset \RR_+^n$.}
\label{alg:game}
\end{algorithm}

The online learning setting proceeds in a series of trials where the algorithm makes a prediction or takes an action associated with an object in the appropriate combinatorial space and then receives the loss of its choice in such a way that the loss of any of the possible combinatorial objects can be easily computed 
(See Prediction Game \ref{alg:game}). 
The algorithm can then update its internal representation based on this feedback and the process moves on to the next trial.
Unlike batch learning settings, there is no assumed distribution from which losses are randomly drawn. Instead the losses are drawn
adversarially.
In general, an adversary can force arbitrarily large loss on the algorithm. So instead of measuring the algorithm's performance by 
the total loss incurred, the algorithm is measured by its \emph{regret}, the amount of loss the algorithm incurs above that of the single
best predictor in some comparator class.  
Usually the comparator class is the class of objects in the combinatorial space being learned.
To make the setting concrete, consider the case of learning Huffman trees for compression\footnote{Huffman trees \citep{cormen2001introduction} are binary trees which construct prefix codes (called Huffman codes) for data compression. The plaintext symbols are located at the leaves of the tree and the path from the root to each leaf defines the prefix code for the associated symbol.
}. 
In each trial, the algorithm would (perhaps randomly) predict a Huffman tree, and then obtain a sequence of symbols to be encoded. 
The loss of the algorithm on that trial is the average bits per symbol
to encode the sequence using the predicted Huffman tree. 
More generally, the loss could be defined as the inner product of any loss vector from the unit cube and the code lengths of the symbols. 
The total loss of the algorithm is the expected average bits per symbol summed over trials.
The regret of the algorithm is the difference between
its total loss and the sum over trials of the
average bits per symbol for the single best Huffman tree chosen in hindsight.
Therefore the regret of the algorithm can be viewed as the cost of not knowing the best combinatorial object ahead of time.
With proper tuning, the regret is typically logarithmic in the number of combinatorial objects.

One way to create algorithms for these combinatorial problems is to use one of the well-known so-called ``experts algorithms'' 
like Randomized Weighted Majority \citep{littlestone1994weighted} or Hedge \citep{freund1997decision} with each combinatorial object is treated as an ``expert''. 
However, this requires explicitly keeping track of one weight for each of the exponentially many combinatorial objects, and
thus results in an inefficient algorithm. Furthermore, it also causes an additional loss range factor in the regret bounds as well.  
There has been much work on creating efficient algorithms that implicitly encode the weights over the set of combinatorial objects using concise representations.
For example, many distributions over the $2^n$ subsets of $n$ elements can be encoded by the probability of including each of the $n$ elements.
In addition to subsets, such work includes permutations \citep{helmbold2009learning, yasutake2011online, ailon2014improved}, paths \citep{takimoto2003path, kuzmin2005optimum}, and $k$-sets \citep{warmuth2008randomized}. 

There are also some general tools for learning combinatorial concepts. \cite{suehiro2012online} introduced efficient online learning algorithms with good regret bounds for structures that can be formulated by submodular functions\footnote{For instance, permutations belong to such classes of structures (see \cite{suehiro2012online}); but Huffman trees do not as the sum of the code lengths of the symbols is not fixed. }.
\textit{Follow the Perturbed Leader (FPL)} \citep{kalai2005efficient} is a simple algorithm which adds random perturbation to the cumulative loss of each component, and then predicts with the combinatorial object with minimum perturbed loss.
The Component Hedge algorithm of \cite{koolen2010hedging} is a powerful generic technique when the implicit encodings are suitably simple.

The Component Hedge algorithm works by performing multiplicative updates on the parameters of its implicit representation.
However, the implicit representation is typically constrained to lie in a convex polytope.
Therefore Bregman projections are used after the update to
return the implicit representation to the desired polytope.
A limitation of Component Hedge is its projection step which is generally only computationally efficient when there are a small (polynomial)
number of constraints on the implicit representations.

The problem of concisely specifying the convex hulls of complicated combinatorial structures (e.g.~permutations and Huffman trees) using few constraints has been well studied in the combinatorial optimization literature. A powerful technique -- namely \textit{extended formulations} -- has been developed to represent these polytopes as a linear projection of a higher-dimensional polyhedron so that the polytope description has far fewer (polynomial instead of exponential)  constraints \citep{kaibel2013constructing, kaibel2011extended, conforti2010extended}.

\paragraph{Contributions:} The main contributions of this paper are:
\begin{enumerate}
\item \underline{The introduction of extended formulation techniques to the machine learning community.} In\\
 particular, the fusion of Component Hedge with extended formulations results in a new methodology for 
designing efficient online algorithms for complex classes of combinatorial objects. 
Our methodology uses a redundant representation for the combinatorial objects where one part of the representation 
allows for a natural loss measure while another enables the simple specification of the class using only polynomially many constraints.
We are unaware of previous online learning work exploiting this kind of redundancy.
To better match the extended formulations to the machine learning applications, we augment the extended formulation with slack variables.  

\item 
\underline{A new and faster prediction technique.} 
Component Hedge applications usually predict by
first re-expressing the algorithm's weight or usage vector as a small convex combination of combinatorial objects, and 
then randomly sample from the convex combination.  
The redundant representation often allows for a more direct and efficient way to generate the algorithm's 
random prediction, bypassing the need to create convex combinations.  
This is always the case for extended formulations based on ``reflection relations'' (as in permutations and Huffman Trees). 
\item  \underline{A new and elegant initialization method.} 
Component Hedge style loss bounds depend on the distance from the initial hypothesis to the best predictor in the class, and a roughly uniform initialization is usually a good choice.  
The initialization of the redundant representation is more delicate.  
Rather than directly picking a feasible initialization, we introduce the idea of first creating an infeasible encoding
with good distance properties, and then projecting it into the feasible polytope. 
This style of implicit initialization improves bounds in some existing work (e.g.~saving a $\log n$ factor in \cite{yasutake2011online}) and has been used to good effect in another domain \citep{rahmanian2017onlinedp}.
\end{enumerate}

\paragraph{Paper Outline:}
Section \ref{sec:related-works} contains an overview of the Component Hedge algorithm
and extended formulations.
Section \ref{sec:method} explains our methodology. 
We then explore the concrete application of our method on Huffman trees and permutation using extended formulations constructed by reflection relations in Section \ref{sec:instantiations}.
Section \ref{sec:fast-prediction} describes our fast prediction technique in the case of using reflection relations.
Finally, Section \ref{sec:conclusion} concludes with contrasting our bounds with those of FPL \citep{kalai2005efficient}, Hedge \citep{freund1997decision} and OnlineRank \citep{ailon2014improved} and describing directions for future work.
The Appendix \ref{app:notations} contains a summary of our notations.

\section{Background}
\label{sec:related-works}
Online learning is a rich and vibrant area, see \cite{cesa2006prediction} for a textbook treatment.
The implicit representations for structured concepts (sometimes called `indirect representations') have been used for a variety 
of problems \citep{helmbold2002direct, helmbold1997predicting, maass1998efficient, takimoto2002predicting, takimoto2003path, yasutake2011online, koolen2010hedging}.
Recall from the Prediction Game \ref{alg:game}, that $t \in \{1..T\}$ is the trial index, $\Hcal$ is the class of combinatorial objects, $\hvechat_t \in \Hcal$ is the algorithm's selected object, and $\ellvec_t$ is the loss vector revealed by the adversary.

\paragraph{Component Hedge:} \cite{koolen2010hedging} developed a generic framework called \textit{Component Hedge} which results in efficient and effective online algorithms over combinatorial objects in $\RR_+^n$ with linear loss. 
Component Hedge maintains a ``usage'' vector $\vvec$ in the polytope $\Vcal$ which is the convex hull of all objects in the combinatorial class $\Hcal$. 
In each trial, the weight of each component (i.e.~coordinate) of $\vvec$ is updated multiplicatively by its associated exponentiated loss $v_i \leftarrow v_i \, e^{- \eta \, \ell_i}$ for all $i \in \{1..n\}$. 

Then the weight vector $\vvec$ is projected back to the polytope $\Vcal$ via relative entropy projection. 
$\Vcal$ is often characterized with a set of equality constraints (i.e.~intersection of affine subspaces).
Iterative Bregman projection \citep{bregman1967relaxation} is often used;  it enforces each constraint in turn.  
Although this can violate previously satisfied constraints, repeatedly cycling through them is 
guaranteed to converge to the proper projection if all the facets of the polytope are equality constraints. 

Finally, to sample with the same expectation as the usage vector, the usage vector is decomposed into corners of the polytope $\Vcal$. Concretely, $\vvec$ is written as a convex combination of some objects in $\Hcal$ using a greedy approach which zeros out at least one component in each iteration.  
%

Component Hedge relies heavily on an efficient characterization of the polytope $\Vcal$ both for projection and decomposition. 
If directly characterizing the polytope $\Vcal$ is either difficult or requires exponentially many facets, Component Hedge cannot be directly applied. In those cases, we show how extended formulations can help with efficiently describing the polytope $\Vcal$.

\paragraph{Extended Formulations:} Many classes of combinatorial objects have polytopes whose discription requires exponentially many facets in the original space (e.g.~see \cite{maurras2010convex}). 
This has triggered the search for more concise descriptions in alternative spaces. 
In recent years, the combinatorial optimization community has given significant attention to
the technique of \textit{extended formulation} where difficult polytopes are represented as a
 linear projection of a higher-dimensional polyhedron \citep{kaibel2013constructing, kaibel2011extended, conforti2010extended}. 
There are many complex combinatorial objects whose associated polyhedra can be described 
as the linear projection of a much simpler, but higher dimensional, polyhedra
(see Figure \ref{fig:ext-form-general}). 

Concretely, assume a polytope $\Vcal \subset \RR_+^n$ is given and described
with exponentially many constraints 
as $\Vcal = \{ \vvec \in \RR^n_+ \mid B \vvec \leq \dvec \}$
%
in the original space $\RR_+^n$ . We assume that using some additional variables $\xvec \in \RR_+^m$, $\Vcal$ can be written efficiently as
\begin{equation}
\Vcal = \{ \vvec \in \RR^n_+ \mid \exists \xvec \in \RR^m_+ : C \vvec + D \xvec \leq \fvec \} \label{eq:V_XF}\footnote{Note that for each $\vvec \in \Vcal$ there exists a $\xvec \in \Xcal$, but it is not necessarily unique.}
\end{equation}
with $r = \text{poly}(n)$ constraints.  Vector $\xvec \in \RR_+^m$ is an e\textbf{x}tended formulation\footnote{Throughout this paper, w.l.o.g., we assume $\xvec$ is in positive quadrant of $\RR^n$, since an arbitrary point in $\RR^n$ can be written as $\xvec = \xvec^+ - \xvec^-$ where $\xvec^+, \xvec^- \in \RR_+^n$.} belonging to the set
\begin{equation}
\Xcal = \{ \xvec \in \RR^m_+ \mid \exists \vvec \in \RR^n_+ : C \vvec + D \xvec \leq \fvec \} \label{eq:X_XF}
\end{equation}
Extended formulations incur the cost of additional variables for the benefit of a simpler (although, higher dimensional) polytope.

 \bigskip
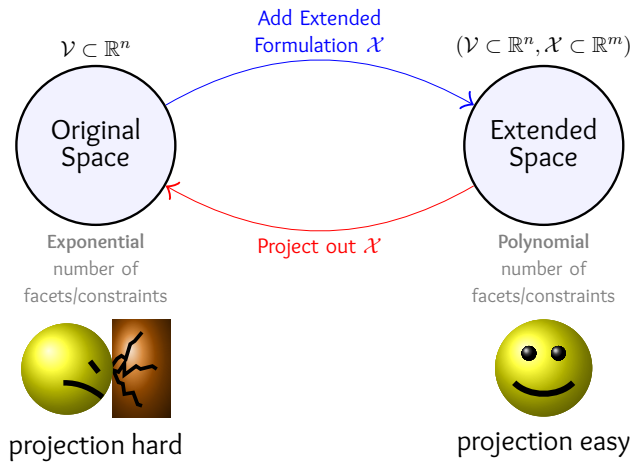
\begin{figure}
\centering
\scalebox{.75}{

\tikzset{
	bigArrow/.style={
		decoration={markings,mark=at position 1 with {\arrow[scale=3]{>}}},
		postaction={decorate},
		shorten >=0.4pt}
	}

\tikzstyle{space}=[draw, circle, very thick, black, text width = 0.8, align=center, fill=blue!5]

{ \overlock 
\begin{tikzpicture}

\node[space, label={${\cal V} \subset \RR^{n}$}, %
	 label={[text width = 1.15in, align=center, gray]below:{\small  \textbf{Exponential} number of facets/constraints} } ] %
	(orig) [text width = 0.9in, align=center]  {\Large Original Space} ;

\node[space, label={ $({\cal V} \subset \RR^n, {\cal X} \subset \RR^m)$ }, %
	label={[text width = 1.15in, align=center, gray]below:{\small  \textbf{Polynomial} number of facets/constraints} } ] %
	(ext) [right = 2.0in of orig, text width = 0.9in, align=center] {\Large Extended Space} ;

\node[below = 0.6in of orig, label=below:{\overlock \Large projection hard}] {\dWalley[6]} ;

\node[below = 0.6in of ext, label=below:{\overlock \Large projection easy}] {\dSmiley[6]};

\draw[bigArrow, bend left, blue] (orig) to node[auto, blue, text width = 1.0in, align=center] 
			{Add Extended Formulation $\Xcal$} (ext);
			
\draw[bigArrow, bend left, red] (ext) to node[auto, red, text width = 1.0in, align=center] 
			{Project out $\Xcal$} (orig);

\end{tikzpicture}

}}
    \caption{Extended formulation.}
    \label{fig:ext-form-general}
    
\end{figure}

\section{The Method}
\label{sec:method}

Here we describe our general methodology for 
using extended formulations to develop new learning algorithms.
Consider a class $\Hcal$ of combinatorial objects and its convex hull $\Vcal$.
We assume there is no efficient description of $\Vcal$ in $\RR_+^n$, but it can be efficiently characterized via an extended formulation $\xvec \in \Xcal$ as in Equations (\ref{eq:V_XF}) and (\ref{eq:X_XF}). 


As described in Section \ref{sec:related-works}, in order to apply Component Hedge (especially the projection), we need to have equality constraints instead of inequality ones. 
Thus, we introduce slack variables $\svec \in \RR_+^{r}$, where $r$ is the number of constraints. Equation (\ref{eq:V_XF}) now becomes
\begin{equation}
\Vcal = \{ \vvec \in \RR^n_+ \mid \exists \xvec \in \RR^m_+, \svec \in \RR_+^{r} : C \vvec + D \xvec + \svec =  \fvec \} \nonumber
\end{equation}
Now, in order to keep track of a usage vector $\vvec \in \Vcal$, we use the following novel representation:
\begin{equation}
\Wcal = \{ \underbrace{(\vvec, \xvec, \svec)}_{\wvec} \in \RR^{n+m+r}_+ \mid  C \vvec + D \xvec + \svec =  \fvec \} \nonumber
\end{equation}
where $\Wcal$ is characterized by $r$ affine constraints. We refer to $\Wcal$ as \textit{the augmented formulation}.
Observe that, despite potential redundancy in representation, all three constituents are useful in this new encoding: $\vvec$ is needed to encode the right loss, $\xvec$ is used for efficient description of the polytope, and $\svec$ is incorporated to have equality constraints. 

\subsection{XF-Hedge Algorithm}
\label{sec:xf-hedge}

Having developed the well-equipped space $\Wcal$, Component Hedge can now be applied.
Since $\vvec$ is the only constituent over which the loss vector $\ellvec^t$ is defined, we work with $\Lvec^t = (\vec{\ell}^t, \zero, \zero) \in [0,1]^{n+m+r}$ in the augmented formulation space $\Wcal$. We introduce a new type of Hedge algorithm combined with e\underline{x}tended \underline{f}ormulation -- \textit{XF-Hedge} (See Algorithm \ref{alg:xf-hedge}). Similar to Component Hedge, XF-Hedge consists of three main steps: \textit{Prediction}, \textit{Update}, and \textit{Projection}.

\begin{algorithm}
\begin{algorithmic}[1]
\State $\wvec^0 = (\vvec^0, \xvec^0, \svec^0) \in \Wcal$ -- a proper prior distribution discussed in \ref{sec:bounds}
\State For $t = 1, \ldots, T$
\State \quad Set $\hvechat^{t-1} \gets$ \textbf{Prediction}$(\wvec^{t-1})$ where $\hvechat^{t-1} \in \Hcal$ is a random object s.t. $\EE \sbr{\hvechat^{t-1}} = \vvec^{t-1}$
\State \quad Incur a loss $\hvechat^{t-1} \cdot \vec{\ell}^t$
\State \quad \textbf{Update}:
\State \qquad Set $\vtil^{t-1}_i \gets v^{t-1}_i \, e^{- \eta \, \ell^t_i}$ for all $i \in [n]$
\State \quad Set $\wvec^{t} \gets$ \textbf{ Projection}$\underbrace{(\vvectil^{t-1}, \xvec^{t-1}, \svec^{t-1})}_{\wvectil^{t-1}} $ where $\wvec^{t} = \underset{\wvec \in \Wcal}{\arg\min} \; \Delta\left(\wvec || \wvectil^{t-1}\right)$
\end{algorithmic}
\caption{XF-Hedge}
\label{alg:xf-hedge}
\end{algorithm}

\paragraph{Prediction:} Randomly select an object $\hvechat^{t-1}$ from the combinatorial class $\Hcal$ in such a way that $\EE \sbr{\hvechat^{t-1}} = \vvec^{t-1}$. The details of this step depend on the combinatorial class $\Hcal$ and the extended formulation used for $\Wcal$. In Component Hedge and similar algorithms \citep{helmbold2009learning, koolen2010hedging, yasutake2011online, warmuth2008randomized}, this step is usually done by decomposing\footnote{Note that according to Caratheodory's theorem, such decomposition exists in $\Wcal$ using at most $n+m+r+1$ objects (i.e.~corners of the polytope $\Wcal$).} $\vvec^{t-1}$ into a convex combination of objects in $\Hcal$. In Section \ref{sec:fast-prediction}, we present a faster prediction method for combinatorial classes $\Hcal$ whose extended formulation is constructed by reflection relations.

\paragraph{Update:} Having defined $\Lvec^t = (\vec{\ell}^t, \zero, \zero)$, the updated $\wvectil^{t-1}$ is obtained using a trade-off between the linear loss and the unnormalized relative entropy \citep{koolen2010hedging}:
\begin{displaymath}
\wvectil^{t-1} 
= \underset{\wvec \in \RR^{r}}{\arg\min} \; \Delta(\wvec || \wvec^{t-1}) 
+ \eta \, \wvec \cdot \Lvec^t, \quad \text{ where } \quad \Delta(\avec || \bvec) = \sum_i a_i \log \frac{a_i}{b_i} + b_i - a_i
\end{displaymath}
Using Lagrange multipliers, it is fairly straight-forward to see that only the $\vvec$ components of $\wvec^{t-1}$ are updated:

\[ 
\forall i \in \{1..n\}, \; \vtil^{t-1}_i = v^{t-1}_i \, e^{- \eta \, \ell_i^t} ; 
\qquad \xvectil^{t-1}=\xvec^{t-1} ;
\qquad \svectil^{t-1} = \svec.
\]


\paragraph{Projection:} 
We use an unnormalized relative entropy Bregman projection 
to project $\wvectil^{t-1}$ back into $\Wcal$ obtaining 
the new $\wvec^t$ for the next trial.
\begin{equation}
\wvec^{t} 
= \underset{\wvec \in \Wcal}{\arg\min} \; \Delta(\wvec || \wvectil^{t-1})  \label{eq:proj}
\end{equation}

Let $\Psi_0, \ldots, \Psi_{r-1}$ be the $r$ hyperplanes where 
the $r$ constraints of $C \vvec + D \xvec + \svec = \fvec$
are satisfied, i.e.~$\wvec \in \Psi_k$ if and only if $\wvec$ satisfies the $k$th constraint.  
Then $\Wcal$ is the intersection of the $\Psi_k$'s.  
Since the non-negativity constraints are already enforced by the definition of $\Delta(\cdot || \cdot)$, 
it is possible to solve (\ref{eq:proj}) using iterative Bregman projections\footnote{In \cite{helmbold2009learning} Sinkhorn balancing is used for projection which is also a special case of iterative Bregman projection.} \citep{bregman1967relaxation}. Starting from $\pvec_{0} = \wvectil^{t-1}$, we iteratively compute:
\begin{equation}
\pvec_{k} = 
\underset{\pvec \in \Psi_{(k \bmod r)}}{\arg\min} \; \Delta(\pvec || \pvec_{k-1}) \nonumber
\end{equation}
%
%
repeatedly cycling through the constraints.
In Appendix \ref{sec:proj-to-constraint}, we discuss how one can efficiently project onto each hyperplane $\Psi_k$ for all $k \in \{1..r\}$. 
It is known that $\pvec_{k}$ converges in norm to the unique solution of (\ref{eq:proj}) \citep{bregman1967relaxation, bauschke1997legendre}.

\subsection{Regret Bounds}
\label{sec:bounds}
Similar to Component Hedge, the general regret bound depends on the initial weight vector $\wvec^0 \in \Wcal$ via $\Delta( \wvec(\hvec) || \wvec^0)$ where $\wvec(\hvec) \in \Wcal$ is the augmented formulation of the object $\hvec \in \Hcal$ against which the algorithm is compared (the best $\hvec$ for the adversarially chosen sequence of losses).
\begin{lemma}
\label{lemma:ch-bounds}
Let $L^* := \underset{\hvec \in \Hcal}{\min} \sum_{t=1}^T \hvec \cdot \vec{\ell}^t$. By proper tuning of the learning rate $\eta$:
\begin{align*}
\EE \sbr{\sum_{t=1}^T \hvechat^{t-1} \cdot \vec{\ell}^t } 
- \underset{\hvec \in \Hcal}{\min} \sum_{t=1}^T \hvec \cdot \vec{\ell}^t
\leq
\sqrt{2 L^* \, \Delta(\wvec(\hvec) || \wvec^0)  } + \Delta(\wvec(\hvec) || \wvec^0)
\end{align*}
\end{lemma}

The proof uses standard techniques from the online learning literature (see, e.g.~, \citep{koolen2010hedging}) and is given in Appendix \ref{sec:proof:lemma:ch-bounds}. 
In order to get good bounds, the initial weight $\wvec^0$ must be ``close'' to all corners $\hvec$ of the polytope, and thus in the``middle" of $\Wcal$.
In previous works \citep{koolen2010hedging, yasutake2011online, helmbold2009learning}, the initial weight is explicitly chosen and it is often set to be the uniform usage of the objects. 
This explicit initialization approach may be difficult to perform when the polytope has a complex structure.

Here, instead of explicitly selecting $\wvec^0 \in \Wcal$, we implicitly design the initial point. 
First, 
we find an intermediate ``middle'' point $\wvectil \in \RR^{n+m+r}$ with good distance properties,
and then project $\wvectil$ into $\Wcal$ to obtain the initial $\wvec^0$ for the first trial.

A good choice for $\wvectil$ is $U \, \one$ where $\one \in \RR^{n+m+r}$ is the vector of all ones, and $U \in \RR_+$ is an upper-bound on the infinity norms of the corners of polytope $\Wcal$. 
This leads to the nice bound $\Delta(\wvec(\hvec) || \wvectil) \leq (n+m+r) U$ 
for all objects $\hvec \in \Hcal$.  
The Generalized Pythagorean Theorem \citep{herbster2001tracking}
ensures that the same bound holds for $\wvec^0$ 
(see Appendix \ref{sec:proof:lemma:xf-init} for the details).


\begin{lemma}
\label{lemma:xf-init}
Assume that there exists $U \in \RR_+$ such that $\| \wvec(\hvec) \|_\infty \leq U$ for all $\hvec \in \Hcal$. Then the initialization method finds a $\wvec^0 \in \Wcal$ such that for all $\hvec \in \Hcal$,  
$\Delta(\wvec(\hvec) || \wvec^0) \leq (n+m+r) \, U$.
\end{lemma}

Combining Lemmas \ref{lemma:ch-bounds}, and \ref{lemma:xf-init} gives the following guarantee.

\begin{theorem}
\label{thm:bounds}
If each $\ellvec^t \in [0,1]^n$ and 
$\| \wvec(\hvec) \|_\infty \leq U$ for all $\hvec \in \Hcal$, then XF-hedge's regret is:
\begin{align*}
\EE \sbr{\sum_{t=1}^T \hvechat^{t-1} \cdot \vec{\ell}^t } 
- \underset{\hvec \in \Hcal}{\min} \sum_{t=1}^T \hvec \cdot \vec{\ell}^t
\leq
\sqrt{2 L^* \,(n+m+r) \, U  } + (n+m+r) \, U
\end{align*}
\end{theorem}


\section{XF-Hedge Examples Using Reflection Relations}
\label{sec:instantiations}
One technique for constructing extended formulations is called \emph{reflection relations} \citep{kaibel2013constructing}, and this technique can be used to efficiently describe the polytopes of permutations and Huffman trees. 
Here we describe how reflection relations can be used with the XF-Hedge framework to create concrete learning algorithms for permutations and Huffman trees.


As in \cite{yasutake2011online} and \cite{ailon2014improved}, we consider losses that are linear in the first order representation of the objects \citep{diaconis1988group}. 
For permutations of  $n$ items, the first order representation is vectors $\vvec \in \RR^n$ where each of the elements of $\{1,2,\ldots, n \}$ appears exactly once\footnote{In contrast, \cite{helmbold2009learning} work with the second order representation (i.e.~Birkhoff polytope), and consequently losses, which is a more general loss family (see \cite{yasutake2011online} for comparison).} 
and for Huffman trees on $n$ symbols, the first order representation is vectors $\vvec \in \RR^n$ where each $v_i$ is an integer indicating the depth of the leaf corresponding to symbol $i$ in the coding tree. 
At each trial the loss is $\vvec \cdot \ellvec$ where the adversary's
$\ellvec$ is a loss vector in the unit cube $[0,1]^n$.
This type of loss is sufficiently rich to capture well-known natural losses like \textit{average code length} 
for Huffman trees (when $\ell$ is the symbol frequencies)
and \textit{sum of completion times}  for permutations\footnote{To easily encode the sum of completion times, the
predicted permutation represents the \emph{reverse} order in which the tasks are to be executed.}
(when $\ell$ is the task completion times).


\paragraph{Constructing Extended Formulations from Reflection Relations}

\cite{kaibel2013constructing} show how to construct polynomial size extended formulations
using a canonical corner of the polytope and a fixed sequence of hyperplanes.
These have the property that any corner of the desired polytope can be generated by reflecting the canonical corner through a subsequence of the hyperplanes. 
These reflections are \textit{one-sided} in the sense that they map the half-space containing the canonical corner to the other half-space.
For example, the corners of Figure \ref{fig:reflection-relation} (Left) can be generated in this way.
Of course the hard part is to find a good sequence of hyperplanes with this property.


A key idea for generating the entire polytope 
is to allow ``partial reflections''  where the point to be reflected can not just be kept (skipping the reflection)
or replaced by its reflected image, but mapped to any point on the line segment joining the point and its reflected image as illustrated in Figure \ref{fig:reflection-relation} (Right). 
Since any point in the convex hull of the polytope can be constructed by at least one sequence of partial one-sided reflections, 
every point in the polytope has an alternative representation in terms of how much each reflection was used 
to generate it from the canonical corner (see Figure \ref{fig:reflection-relation}).
Each of these parameterized partial reflections is a \emph{reflection relation},


For each reflection relation, there will be one additional variable indicating the extent to which the reflection occurs, and two additional inequalities for the extreme cases of complete reflection and no reflection.
Therefore, if the polytope can be expressed with polynomially many reflection relations, then it has an extended formulation of polynomial size with polynomially many constraints. 
Appendix \ref{sec:inductive-ext-form} provides more details about the type of results shown by \cite{kaibel2013constructing}.

\def\x{5}
\def\y{5}
\def\k{1.7}
\def\e{.5}
\def\ee{0}
\def\f{.6}
\def\kk{1.5}
\def\xx{3.5}
\def\a{.9}
\def\b{3}
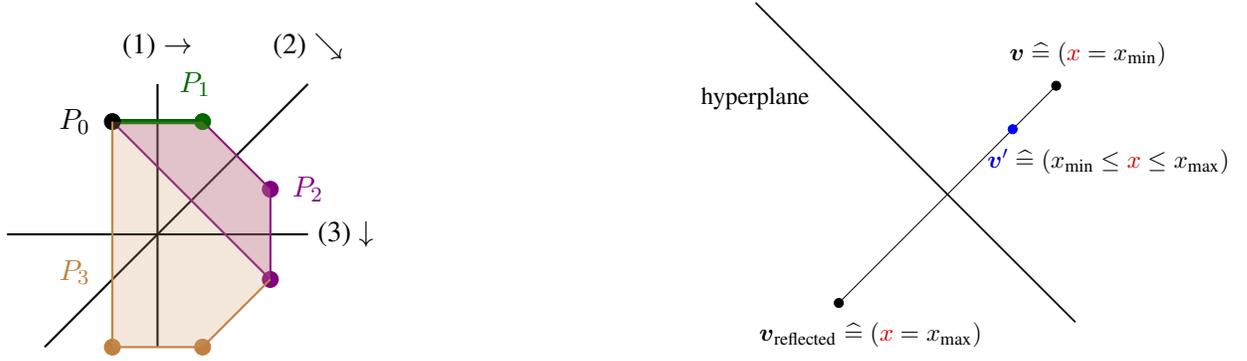
\begin{figure}[!tbp]
  \centering
  \begin{minipage}[b]{0.4\textwidth}
  \scalebox{1.0}{
    \begin{tikzpicture}
    \draw[thick] (\kk,0) -- (\kk,\xx);
    \node[] at (\kk,\xx + \e) {(1) $\rightarrow$};

    \draw[thick] (0,0) -- (\xx,\xx);
    \node[] at (\xx,\xx + \e) {(2) $\searrow$};
    
    \draw[thick] (2*\kk-\b-\e,\kk+\ee) -- (\xx,\kk+\ee);
    \node[] at (\xx+\e,\kk+\ee) {(3) $\downarrow$};

    \filldraw[violet] (\b, 2*\kk-\a) circle (3pt);     
    \filldraw[violet] (\b,\a) circle (3pt);
    
    \draw[thick,violet] (\a,\b) -- (\b,\a);
    \draw[thick,violet] (2*\kk-\a, \b) -- (\b, 2*\kk-\a);
    \draw[thick,violet] (\b,\a) -- (\b, 2*\kk-\a);

    \node[] at (\b + \e, 2*\kk-\a) {\textcolor{violet}{$P_2$}};
  
  \fill[violet!75,nearly transparent] (\a,\b) -- (2*\kk-\a,\b) -- (\b, 2*\kk-\a) -- (\b,\a) -- cycle;
    \filldraw[black!60!green] (2*\kk-\a,\b) circle (3pt); 
    \draw[line width=0.7mm,black!60!green] (\a,\b) -- (2*\kk-\a,\b);
    \node[] at (\kk + \e, \b + \e) {\textcolor{black!60!green}{$P_1$}};

    \filldraw (\a,\b) circle (3pt);
    \node[] at (\a - \e, \b) {$P_0$};
    
   \filldraw[brown] (2*\kk-\a, 2*\kk+2*\ee-\b) circle (3pt);     
   \filldraw[brown] (\a, 2*\kk+2*\ee-\b) circle (3pt);     

   \draw[thick,brown] (\b, \a) -- (2*\kk-\a, 2*\kk+2*\ee-\b);
   \draw[thick,brown] (2*\kk-\a, 2*\kk+2*\ee-\b) -- (\a, 2*\kk+2*\ee-\b);
   \draw[thick,brown] (\a, 2*\kk+2*\ee-\b) -- (\a,\b);
   
   \fill[brown!75,nearly transparent] (\a,\b) -- (2*\kk-\a,\b) -- (\b, 2*\kk-\a) -- (\b,\a) -- (2*\kk-\a, 2*\kk+2*\ee-\b) -- (\a, 2*\kk+2*\ee-\b)  -- cycle;
   
   \node[] at (\a-\e, \kk-\e) {\textcolor{brown}{$P_3$}};

%
%
%

\end{tikzpicture} }   
  \end{minipage}
\hfill
  \begin{minipage}[b]{0.4\textwidth}
    \scalebox{.85}{

    \begin{tikzpicture}
    
    \draw[thick] (0,\y) -- (\x,\y-\x);
    \node[] at (0,\y-3*\e) {hyperplane};

    \draw (\x * \f -\k, \y - \x *\f - \k) -- (\x * \f+\k, \y-\x * \f+\k);
    
    \filldraw (\x* \f+\k, \y-\x* \f+\k) circle (2pt);
    \node[] at (\x* \f+\k+\e, \y-\x* \f+\k+\e) {$\vvec \; \widehat{=} \; (\textcolor{new_red}{x} = x_\text{min})$};
    
    \filldraw (\x* \f-\k, \y - \x* \f - \k) circle (2pt);
    \node[] at (\x* \f-\k+\e, \y - \x* \f - \k -\e) {$\vvec_\text{reflected} \; \widehat{=} \;  (\textcolor{new_red}{x} = x_\text{max})$};
    
    \filldraw[blue] (\x* \f+\k* \f, \y-\x* \f+\k* \f) circle (2pt);
    \node[] at (\x* \f+\k* \f +3*\e, \y-\x* \f+\k* \f-\e) {$\textcolor{new_blue}{\vvec'} \; \widehat{=} \;  (x_\text{min} \leq \textcolor{new_red}{x} \leq x_\text{max})$};

\end{tikzpicture}
}
  \end{minipage}
\caption{(Left) The 6 corners of the polytope are generated by 
subsequences of one-sided reflections through lines (1), (2), and (3),
starting from the canonical point $P_0$. 
Using partial reflections, we can generate the entire polytope.
(Right) A partial reflection of $\vvec$ to $\textcolor{new_blue}{\vvec'}$ corresponds to ($\widehat{=}$) a variable $\textcolor{new_red}{x}$ indicating how far $\textcolor{new_blue}{\vvec'}$ moves towards $\vvec$'s image $\vvec_\text{reflected}$. }
\label{fig:reflection-relation}
\end{figure}



\paragraph{Extended Formulations for Objects Closed under Re-Ordering}

Assume we want to construct an extended formulation for a class of combinatorial objects which is closed under any re-ordering (both Huffman trees and  permutations both have this property).
Then reflection relations corresponding to swapping pairs of elements are useful.
Swapping elements $i$ and $j$ can be implemented with a hyperplane going through the origin 
and having normal vector $\evec_i - \evec_j$ (here $\evec_i$ is the $i$th unit vector). 
The identity permutation is the natural canonical corner, 
so the one-sided reflections are only used for $\vvec$ where $v_i \leq v_j$.

Implementing the reflection relation for the $i$, $j$ swap  
uses an additional variable along with two additional inequalities. 
Concretely, assume $\vvec \in \RR^n$ is going into this reflection relation and $\vvec' \in \RR^n$ is the output, 
so $\vvec'$ is in the convex combination of $\vvec$ and its reflection.  
It is natural to encode this as $\vvec' = \gamma \vvec + (1-\gamma) \vvec_\text{reflected}$.
However, we found it more convenient to parameterize $\vvec'$ by its absolute distance $x$ from $\vvec$, rather than the relative distance $\gamma \in [0,1]$.
Using this parameterization, we have
$\vvec' = \vvec + x \, (\evec_i - \evec_j)$ 
constrained by $(\evec_i - \evec_j) \cdot \vvec \leq (\evec_i - \evec_j) \cdot \vvec' \leq  - (\evec_i - \evec_j) \cdot \vvec$. 
Therefore the possible relationships between between $\vvec'$ and $\vvec$ 
can be encoded with the additional variable $x$ 
and the following constraints:\footnote{In general $\vvec$ (and thus $v_j$ and $v_i$) may be functions of the variables for previous reflection relations.}
\begin{equation}
\label{eq:single-reflection-formulation}
\textcolor{new_blue}{\vvec'}  = \mvec \, \textcolor{new_red}{x} + \vvec \quad \text{where }
\mvec = \evec_i - \evec_j,
\qquad 0 \leq \textcolor{new_red}{x} \leq v_j - v_i .
\end{equation}


Notice that $x$ indicates the amount of change in the $i$th and $j$th elements which can go from zero (remaining unchanged) to the maximum swap capacity $v_j - v_i$. 

Suppose the desired polytope is described using $m$ reflection relations and with canonical point $\cvec$.
Then starting from $\cvec$ and successively applying the equation in (\ref{eq:single-reflection-formulation}), 
we obtain the connection between the extended formulation space $\Xcal$ and original space $\Vcal$:
\begin{displaymath}
\textcolor{new_blue}{\vvec} = M \, \textcolor{new_red}{\xvec} + \cvec, \quad \textcolor{new_blue}{\vvec}, \cvec \in \Vcal \subset \RR^n, \; \textcolor{new_red}{\xvec} \in \Xcal \subset \RR^m, \; M \in \{-1,0,1\}^{n \times m}.
\end{displaymath}



\cite{kaibel2013constructing} showed that the $m$ reflection relations corresponding to the $m$ comparators in an arbitrary $n$-input 
sorting network\footnote{A \textit{sorting network} is a sorting algorithm where the comparisons are fixed in advance. See e.g.~\citep{cormen2001introduction}.}
 generates the permutation polytope (see Figure \ref{fig:ext-form-ex}).
A similar extended formulation for Huffman trees can be built using an arbitrary sorting network along with $O(n \log n)$ additional comparators and simple linear maps (which do not require extra variables) and the canonical corner $\cvec = [1,2, \ldots, n-2, n-1, n-1]^T$ (see Section 2.24 in \cite{pashkovich2012extended} for more details).
%
%
%
Note that the reflection relations are applied in reverse order than their use in the sorting network 
(see Figure \ref{fig:ext-form-ex}).

\def\h{.7}
\begin{figure}
    \centering
    \begin{tikzpicture}
\foreach \l in {11.5}{
\foreach \i in {0,5,\l}{
  \node[] at (\i+2, 2.7+\h) {$\textcolor[rgb]{0,.6,0}{\xrightarrow{\text{Sorting Network}}}$};
  \node[] at (\i+2, 0) {$\textcolor{new_red}{\xleftarrow{\text{Extended Formulation}}}$};
  \foreach \a in {0,...,2}
    \draw[thick] (\i+0.5,\a+\h) -- (\i+3.5,\a+\h);
  \foreach \x in {{\i+1,2},{\i+1,1},{\i+2,0},{\i+2,1},{\i+3,2},{\i+3,1}}
    \filldraw (\x+\h) circle (1.5pt+\h);
  \draw[thick] (\i+1,1+\h) -- (\i+1,2+\h);
  \draw[thick] (\i+2,0+\h) -- (\i+2,1+\h);
  \draw[thick] (\i+3,1+\h) -- (\i+3,2+\h);
  
  \node[] at (\i+3.7,2.2+\h) {1};
  \node[] at (\i+3.7,1.2+\h) {2};
  \node[] at (\i+3.7,0.2+\h) {3};
}

  \node[] at (.3,2.2+\h) {\textcolor[rgb]{0,0,1}{1}};
  \node[] at (.3,1.2+\h) {\textcolor[rgb]{0,0,1}{2}};
  \node[] at (.3,.2+\h) {\textcolor[rgb]{0,0,1}{3}};
  
\foreach \i in {1.2,2.2}{
  \node[] at (\i+.3,2.2+\h) {\textcolor[rgb]{0,.6,0}{1}};
  \node[] at (\i+.3,1.2+\h) {\textcolor[rgb]{0,.6,0}{2}};
  \node[] at (\i+.3,.2+\h) {\textcolor[rgb]{0,.6,0}{3}};
}
\node[] at (.8,1.5+\h) {\textcolor{new_red}{0}};
\node[] at (1.8,.5+\h) {\textcolor{new_red}{0}};
\node[] at (2.8,1.5+\h) {\textcolor{new_red}{0}};

\node[] at (4.5,1+\h) {and};

\node[] at (5.3,2.2+\h) {\textcolor{new_blue}{3}};
\node[] at (5.3,1.2+\h) {\textcolor{new_blue}{1}};
\node[] at (5.3,.2+\h) {\textcolor{new_blue}{2}};

\node[] at (6.5,2.2+\h) {\textcolor[rgb]{0,.6,0}{1}};
\node[] at (6.5,1.2+\h) {\textcolor[rgb]{0,.6,0}{3}};
\node[] at (6.5,.2+\h) {\textcolor[rgb]{0,.6,0}{2}};

\node[] at (7.5,2.2+\h) {\textcolor[rgb]{0,.6,0}{1}};
\node[] at (7.5,1.2+\h) {\textcolor[rgb]{0,.6,0}{2}};
\node[] at (7.5,.2+\h) {\textcolor[rgb]{0,.6,0}{3}};

\node[] at (5.8,1.5+\h) {\textcolor{new_red}{2}};
\node[] at (6.8,.5+\h) {\textcolor{new_red}{1}};
\node[] at (7.8,1.5+\h) {\textcolor{new_red}{0}};

\node[] at (\l,2.2+\h) {\textcolor{new_blue}{$v_1=2$}};
\node[] at (\l,1.2+\h) {\textcolor{new_blue}{$v_2=1.5$}};
\node[] at (\l,.2+\h) {\textcolor{new_blue}{$v_3=2.5$}};

\node[] at (\l+1.5,2.2+\h) {\textcolor[rgb]{0,.6,0}{1}};
\node[] at (\l+1.5,1.2+\h) {\textcolor[rgb]{0,.6,0}{2.5}};
\node[] at (\l+1.5,.2+\h) {\textcolor[rgb]{0,.6,0}{2.5}};

\node[] at (\l+2.5,2.2+\h) {\textcolor[rgb]{0,.6,0}{1}};
\node[] at (\l+2.5,1.2+\h) {\textcolor[rgb]{0,.6,0}{2}};
\node[] at (\l+2.5,.2+\h) {\textcolor[rgb]{0,.6,0}{3}};

\node[] at (\l+1.6,1.6+\h) {\textcolor{new_red}{$x_3=1$}};
\node[] at (\l+2.6,.6+\h) {\textcolor{new_red}{$x_2=.5$}};
\node[] at (\l+3.6,1.6+\h) {\textcolor{new_red}{$x_1=0$}};
}

\node[] at (9.8,1+\h) {$\xRightarrow{\text{average}}$};
\end{tikzpicture}
    \caption{An extended formulation for permutation on $n=3$ items. 
    The canonical permutation is $[1,2,3]$.
    Elements of \textcolor{new_blue}{$\vvec$} are in \textcolor{new_blue}{blue},
    \textcolor{new_red}{$\xvec$} in \textcolor{new_red}{red}, 
    and the intermediate values are in \textcolor[rgb]{0,.6,0}{green}.    
    }
    \label{fig:ext-form-ex}
\end{figure}
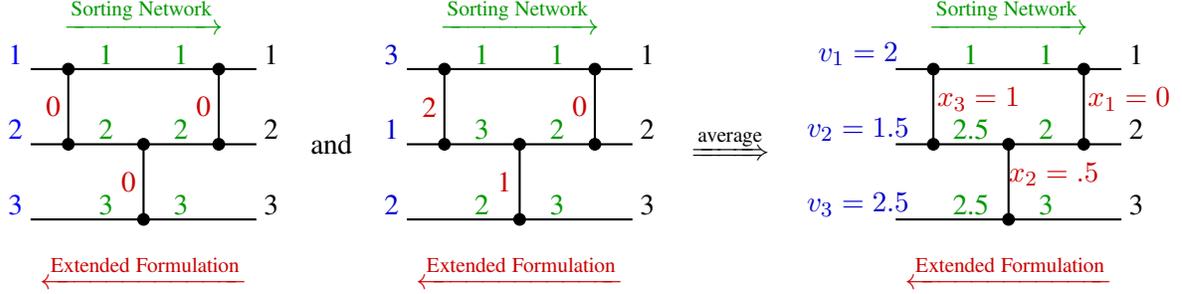


\paragraph{Learning Permutations and Huffman Trees}
As described in the previous subsections, the polytope $\Vcal$ of both permutations and Huffman trees can be efficiently described using $m$ inequality and $n$ equality constraints\footnote{Positivity constraints are excluded as they are already enforced due to definition of $\Delta(\cdot || \cdot)$, and Huffman trees require additional $\Ocal(n \log n)$ inequality constraints beyond those corresponding to the sorting network.}:
\begin{equation}
\Vcal = \{ \vvec \in \RR^n_+ | \exists \xvec \in \RR^m_+ : A \xvec \leq \bvec \text{ and } \vvec = M \xvec + \cvec \} \nonumber
\end{equation}
Adding the slack variables $\svec \in \RR_+^m$, we obtain the augmented formulation $\Wcal$:
\begin{equation}
\Wcal = \{ \wvec = (\vvec, \xvec, \svec) \in \RR^{n+2m}_+ |  A \xvec + \svec = \bvec \text{ and } \vvec = M \xvec + \cvec \} \nonumber
\end{equation}

Note that all the wire values (i.e.~$v_i$'s), as well as $x_i$'s and $s_i$'s are upperbounded by $U=n$.
Using the AKS sorting networks with $m = \Ocal (n \log n)$ comparators \citep{ajtai1983sorting}, we can obtain the regret bounds below from Theorem \ref{thm:bounds}:
\begin{corollary}
XF-Hedge has the following regret bound when learning either permutations or Huffman trees:
\begin{align*}
\EE \sbr{\sum_{t=1}^T \hvechat^{t-1} \cdot \vec{\ell}^t } 
- \underset{\hvec \in \Hcal}{\min} \sum_{t=1}^T \hvec \cdot \vec{\ell}^t
= \Ocal \left(
n \, (\log n)^{\frac{1}{2}} \sqrt{L^*} + n^2 \, \log n
\right)
\end{align*}
\end{corollary}

\section{Fast Prediction with Reflection Relations}
\label{sec:fast-prediction}
From its current
weight vector $\wvec = (\vvec, \xvec, \svec) \in \Wcal$, XF-Hedge randomly selects an object $\hvechat$ from the combinatorial class $\Hcal$ in such a way that $\EE \sbr{\hvechat} = \vvec$. 
In Component Hedge and similar algorithms \citep{helmbold2009learning, koolen2010hedging, yasutake2011online, warmuth2008randomized}, 
this is done by decomposing $\vvec$ into a convex combination of objects in $\Hcal$ followed by sampling. 
In this section, we give a new more direct prediction method for combinatorial classes $\Hcal$ whose extended formulation is constructed by reflection relations.
Our method is faster due to avoiding the decomposition.
 

The values $\xvec$ and $\xvec + \svec$ can be interpreted as amount swapped and the maximum swap allowed for the comparators in the sorting networks, respectively. 
Therefore, it is natural to define $x_i / (x_i + s_i)$ as \textit{swap probability} associated with the $i$th comparator for $i \in \{1..m\}$. 
Algorithm \ref{alg:prediction} incorporates the notion of swap probabilities to construct an efficient sampling procedure from a distribution $\Dcal$ which has the right expectation. 
It starts with the canonical object (e.g.~identity permutation) and feeds it through the reflection relations. 
Each reflection $i$ is taken with probability $x_i / (x_i + s_i)$.
The theorem below (proved in the Appendix \ref{sec:proof_thm_fast-pred}) guarantees the correctness and efficiency of this algorithm.

\begin{theorem}
\label{thm:fast-pred}
(i) Given $(\vvec, \xvec, \svec) \in \Wcal$, Algorithm \ref{alg:prediction} samples $\hvec$ from a $\Dcal$ such that $\EE_\Dcal[\hvec] = \vvec$. \\(ii) The time complexity of Algorithm \ref{alg:prediction} is $\Ocal(m)$.
\end{theorem}

Using the AKS sorting networks \citep{ajtai1983sorting}, 
Algorithm \ref{alg:prediction} predicts in $\Ocal(n \log n)$ time. 
This improves the previously known $\Ocal(n^2)$ prediction procedure for mean-based algorithms\footnote{It also matches the time complexity of prediction step for non-mean-based permutation-specialized OnlineRank \citep{ailon2014improved} and also the general FPL \citep{kalai2005efficient} algorithm both for permutations and Huffman Trees.} for permutations \citep{yasutake2011online, suehiro2012online}.

\begin{algorithm}
\begin{algorithmic}[1]
\State \textbf{Input: } $(\xvec, \svec) \in \RR_+^{2m}$ 
\State \textbf{Output: } A prediction $\hvechat \in \Hcal$ 
\State $\hvechat \gets \cvec$ 
\For{$k=1$ to $m$}
\State $(i_k, j_k) \leftarrow $ wire indices associated with the $k$-th comparator 
\If {$x_i =0$} 
\State \textbf{continue}
\Else
\State Switch the $i_k$th and $j_k$th components of $\hvechat$ w.p. $x_k / (x_k + s_k)$.
\EndIf
\EndFor
\Return $\hvechat$
\end{algorithmic}
\caption{Fast-Prediction}
\label{alg:prediction}
\end{algorithm}

\section{Conclusion and Future Work}
\label{sec:conclusion}
Table~\ref{table:compare-bounds} contains a comparison of the regret bounds for the new XF-Hedge algorithm, 
OnlineRank \citep{ailon2014improved}, Follow the Perturbed Leader (FPL) \citep{kalai2005efficient}, and the Hedge algorithm \citep{freund1997decision} 
which inefficiently maintains an explicit weight for each of the exponentially many permutations or Huffman trees.
%
%
For permutations, the regret bound of general XF-Hedge methodology is within a factor $\sqrt{\log n}$ of the state-of-the-art  algorithm OnlineRank \citep{ailon2014improved}. 
When compared with the generic explicit Hedge algorithm (which is not computationally efficient) and FPL, XF-Hedge has a better loss bound by a factor of $\sqrt{ n }$.

When comparing on Huffman trees, we consider two loss regimes: 
one where the loss vectors are from the general unit cube, and consequently, the per-trial losses are in 
$\Ocal(n^2)$ (like permutations), 
and another where the loss vectors represent frequencies and lie on the unit simplex so the per-trial losses are in $\Ocal(n)$.  
In the first case, as with permutations, XF-Hedge has the best asymptotic bounds.
In the second case, the lower loss range benefits Hedge and FPL, and the regret bounds of all three algorithms match. 

\begin{table} 
\centering
 \begin{tabular}{| l | l | l | l |}
 \hline
 \multirow{2}{*}{Algorithm} &  \multirow{2}{*}{Permutation} & \multicolumn{2}{|c|}{Huffman Tree}\\
\cline{3-4}
   &   & $\vec{\ell}^t \in$ Unit Cube & $\vec{\ell}^t \in$ Unit Simplex \\
 \hline
 XF-Hedge & $\Ocal \begin{pmatrix}
 n  (\log n)^{\frac{1}{2}} \sqrt{L^*}  \\
 + n^2 \log n
\end{pmatrix}$ 
 & $\Ocal \begin{pmatrix}
 n  (\log n)^{\frac{1}{2}} \sqrt{L^*}  \\
 + n^2 \log n
\end{pmatrix}$ 
 & $\Ocal \begin{pmatrix}
 n  (\log n)^{\frac{1}{2}} \sqrt{L^*}  \\
 + n^2 \log n
\end{pmatrix}$  \\
 \hline
 OnlineRank  & {$\Ocal( n \sqrt{L^*} + n^2)$} & {--} & {--}\\
 \hline
  FPL  & $\Ocal \begin{pmatrix}
 n^{\frac{3}{2}} (\log n)^{\frac{1}{2}}  \sqrt{L^*}   \\
 + n^3 \log n
\end{pmatrix}$ 
 &   $\Ocal \begin{pmatrix}
 n^{\frac{3}{2}} (\log n)^{\frac{1}{2}}  \sqrt{L^*}   \\
 + n^3 \log n
\end{pmatrix}$
 &   $\Ocal \begin{pmatrix}
 n  (\log n)^{\frac{1}{2}} \sqrt{L^*}  \\
 + n^2 \log n
\end{pmatrix}$  \\
 \hline
Hedge Algorithm & $\Ocal \begin{pmatrix}
 n^{\frac{3}{2}} (\log n)^{\frac{1}{2}}  \sqrt{L^*}   \\
 + n^3 \log n
\end{pmatrix}$
 &  $\Ocal \begin{pmatrix}
 n^{\frac{3}{2}} (\log n)^{\frac{1}{2}}  \sqrt{L^*}   \\
 + n^3 \log n
\end{pmatrix}$
 &  $\Ocal \begin{pmatrix}
 n  (\log n)^{\frac{1}{2}} \sqrt{L^*}  \\
 + n^2 \log n
\end{pmatrix}$  \\
 \hline
 \end{tabular}
 \caption{Comparing the regret bounds of XF-Hedge with other existing algorithms in different problems and different loss regimes.}
 \label{table:compare-bounds}
\end{table}

In traditional online expert settings, projections are exact (i.e.~renormalizing a weight vector). In contrast, for learning combinatorial objects iterative Bregman projections are used \citep{koolen2010hedging, helmbold2009learning}. These methods 
are known to converge to the exact projection \citep{bregman1967relaxation, bauschke1997legendre}, 
believed to have fast linear convergence \citep{dhillon2007matrix}, 
and empirically are very efficient \citep{koolen2010hedging}. 
However, the iterative nature of the projection step necessitates an analysis to bound the additional loss incurred in the prediction step due to stopping short of full convergence in the projection step. 
Appendix \ref{sec:apprx-breg} provides such analysis for XF-Hedge 
when learning Huffman trees and permutations.

In conclusion, we have presented a general methodology for creating online learning algorithms from extended formulations. 
Our main contribution is the XF-Hedge algorithm that enables the efficient use of Component Hedge techniques 
on complex classes of combinatorial objects. 
Because XF-Hedge is in the Bregman projection family of algorithms, 
many of the tools from the expert setting are likely to carry over. 
This includes 
lower bounding weights for shifting comparators \citep{herbster1998tracking},
long-term memory \citep{bousquet2002tracking}, 
and adapting the updates to the bandit setting \citep{audibert2011minimax}.
Several important areas remain for potentially fruitful future work:

\paragraph{More Applications}
There is a rich literature on extended formulation for different combinatorial objects \citep{conforti2010extended, kaibel2011extended, pashkovich2012extended, afshari2017two, fiorini2013combinatorial}.
Which combinatorial classes have both natural online losses and suitable extended formulations
so XF-Hedge is appropriate?
For instance, building on the underlying ideas of XF-Hedge, \cite{rahmanian2017onlinedp} developed a family of learning algorithms focusing on extended formulations constructed by dynamic programming.

\paragraph{More Complex Losses} 
The redundant representation we introduced can be used to express different losses.
Although our current applications do not assign loss to the extended formulation variables ($\xvec$) and their associated slack variables ($\svec$), 
these additional variables enable the expression of different kinds of losses.
For what natural losses could be these additional variables useful?





\acks{We thank Manfred K. Warmuth for his helpful discussions. We would like to acknowledge support for this project from the National Science Foundation (NSF grant IIS-1619271). }



\bibliographystyle{plain}
\bibliography{alt2018}

\vfill

\pagebreak

\appendix

\label{sec:appendix}
\section{Table of Notations}
\label{app:notations}
\begin{table}[H]
\centering
\begin{tabular}{| c | l |}
\hline
Symbol & Description \\
\hline
$n$ & The dimensionality of the combinatorial object\\
$\Hcal$ & The set of all objects in $\RR_+^n$ \\
$\hvec$ & A particular object in $\Hcal$ \\
$T$ & The number of trials\\
$\ellvec$ & The loss vector revealed by the adversary in $[0,1]^n$\\
$\Vcal$ & The convex hull of all objects in $\Hcal$ \\
$\vvec$ & A point in $\Hcal$ \\
$m$ & The dimensionality of the extended formulation\\
$\xvec$ & A point in extended formulation \\
$\Xcal$ & The space of extended formulations $\xvec$ \\
$\Wcal$ & The augmented formulation \\
$\wvec$ & A point in the augmented formulation $\Wcal$ \\
$\svec$ & The slack vector in the augmented formulation \\
$r$ & The dimensionality of the slack vector \\
$\Delta( \cdot || \cdot)$ & the unnormalized relative entropy\\
~ & i.e.~ $\Delta(\wvec_1 || \wvec_2) = \sum_{i} w_{1,i} \log \frac{w_{1,i}}{w_{2,i}} + w_{2,i} - w_{1,i} $\\
$\wvec(\hvec)$ & A point in $\Wcal$ associated with the object $\hvec$ \\
$U$ & An upper-bound for $\| \wvec(\hvec) \|_\infty$\\
$L^*$ & The cumulative loss of the best object in hindsight \\
~ & i.e.~ $\underset{\hvec \in \Hcal}{\min} \sum_{t=1}^T \hvec \cdot \vec{\ell}^t$\\
$M$ & the $n \times m$ matrix representing the affine transformation \\
~ & corresponding to $m$ reflection relations\\
$\cvec$ & the canonical point in $\Hcal$ e.g.~$[1,2,\ldots, n]^T$ for permutations\\
$A$, $\bvec$ & the $m \times m$ matrix of coefficients and $m$-dimensional vector \\
~ & of constant terms specifying $\Xcal$ along $\xvec \geq \zero$ i.e.~ $A \xvec \leq \bvec$\\
\hline
\end{tabular}
 \caption{Table of notations in the order of appearance in the paper.}
 \label{table:notations}
\end{table}


\section{Proof of Lemma \ref{lemma:ch-bounds}}
\label{sec:proof:lemma:ch-bounds}

\begin{proof}
Assuming $\wvec = (\vvec, \xvec, \svec)$ and $\Lvec = (\vec{\ell}, \zero, \zero)$:
\begin{align*}
(1 - e^{-\eta}) \vvec^{t-1} \cdot \vec{\ell}^t 
&= (1 - e^{-\eta}) \wvec^{t-1} \cdot \Lvec^t 
\leq \sum_{i}  w_i^{t-1} (1 - e^{- \eta \, L^t_i}) \\
&= \Delta(\wvec(\hvec) || \wvec^{t-1}) - \Delta(\wvec(\hvec) || \wvectil^{t-1}) + \eta \, \wvec(\hvec) \cdot \Lvec^t \\
&= \Delta(\wvec(\hvec) || \wvec^{t-1}) - \Delta(\wvec(\hvec) || \wvectil^{t-1}) + \eta \, \hvec \cdot \vec{\ell}^t  \\
&\leq \Delta(\wvec(\hvec) || \wvec^{t-1}) - \Delta(\wvec(\hvec) || \wvec^{t}) + \eta \, \hvec \cdot \vec{\ell}^t \\
\end{align*}

The first inequality is obtained using $1 - e^{- \eta x} \geq (1 - e^{-\eta})x$ for $x \in [0,1]$ as done in \cite{littlestone1994weighted}. The second inequality is a result of the Generalized Pythagorean Theorem \citep{herbster2001tracking}, since $\wvec^{t}$ is a Bregman projection of $\wvechat^{t-1}$ into the convex set $\Wcal$ which contains $\wvec(\hvec)$. By summing over $t=1 \ldots T$ and using the non-negativity of divergences, we obtain:
\begin{align*}
&(1 - e^{-\eta}) \sum_{t=1}^T \vvec^{t-1} \cdot \vec{\ell}^t 
\leq \Delta(\wvec(\hvec) || \wvec^{0}) - \Delta(\wvec(\hvec) || \wvec^{T}) + \eta \, \sum_{t=1}^T \hvec \cdot \vec{\ell}^t \\
&\longrightarrow \EE \sbr{\sum_{t=1}^T \hvec^{t-1} \cdot \vec{\ell}^t } \leq
\frac{ \eta \, \sum_{t=1}^T \hvec \cdot \vec{\ell}^t + \Delta(\wvec(\hvec) || \wvec^{0}) }{1 - e^{- \eta}}
\end{align*}

Let $L^* = \underset{\hvec \in \Hcal}{\min} \sum_{t=1}^T \hvec \cdot \vec{\ell}^t$. We can tune $\eta$ as instructed in Lemma 4 in \cite{freund1997decision}:
\begin{displaymath}
\EE \sbr{\sum_{t=1}^T \hvec^{t-1} \cdot \vec{\ell}^t } 
- \underset{\hvec \in \Hcal}{\min} \sum_{t=1}^T \hvec \cdot \vec{\ell}^t
\leq
\sqrt{2 L^* \, \Delta(\wvec(\hvec) || \wvec^0)  } + \Delta(\wvec(\hvec) || \wvec^0)
\end{displaymath}

\end{proof}

\section{Proof of Lemma \ref{lemma:xf-init}}
\label{sec:proof:lemma:xf-init}
\begin{proof}
Let $\widetilde{\wvec} = U \, \one$ in which $\one \in \RR^{m+n+r}$ is a vector with all ones in its components. Now let $\wvec^0$ be the Bregman projection of $\widetilde{\wvec}$ onto $\Wcal$, that is:
\begin{displaymath}
\wvec^{0} 
= \underset{\wvec \in \Wcal}{\arg\min} \Delta(\wvec || \widetilde{\wvec}) 
\end{displaymath}
Now for all $\hvec \in \Hcal$, we have:
\begin{align*}
\Delta(\wvec(\hvec) || \wvec^{0}) 
&\leq \Delta(\wvec(\hvec) || \widetilde{\wvec})  && \text{Pythagorean Theorem} \\
&= \sum_{i \in \{1..n+m+r\}} \left( \wvec(\hvec) \right)_i \log \frac{\left( \wvec(\hvec) \right)_i}{U} +  U - \left( \wvec(\hvec) \right)_i \\
&\leq \sum_{i \in \{1..n+m+r\}}  U && \left( \wvec(\hvec) \right)_i \leq U \\
&= (n+m+r) \, U 
\end{align*}
\end{proof}

\section{Construction of Extended Formulation Using Reflection Relations}
\label{sec:inductive-ext-form}
Instead of starting with a single corner, one could also consider passing an entire polytope as an input through the sequence of (partial) reflections to generate a new polytope. Using this fact, Theorem 1 in \cite{kaibel2013constructing} provides an inductive construction of higher dimensional polytopes via sequences of reflection relations. Concretely, let $P^n_{\text{obj}}$ be the polytope of a given combinatorial object of size $n$. The typical approach is to properly embed $P^n_{\text{obj}} \subset \RR^n$ into $\Phat^n_{\text{obj}} \subset \RR^{n+1}$, and then feed it through an appropriate sequence of reflection relations as an input polytope in order to obtain an extended formulation for $P^{n+1}_{\text{obj}} \subset \RR^{n+1}$. Theorem 1 in \cite{kaibel2013constructing} provides  sufficient conditions for the correctness of this procedure. Again, if polynomially many reflection relations are used to go from $n$ to $n+1$, then we can construct an extended formulation of polynomial size for $P^n_\text{obj}$ with polynomially many constraints. In this paper, however, we work with batch construction of the extended formulation as opposed to the inductive construction.

\section{Facets Constructed by Reflection Relations}
\label{sec:struc-ext-form}

\begin{lemma}
\label{lemma:A-matrix}
Let $M$ be the matrix representing the affine transformation corresponding to $m$ reflection relations and
\begin{displaymath}
A = \text{Tri}(M^T M) + I, \quad \bvec = - M^T \cvec 
\end{displaymath}
in which Tri$(\cdot)$ is a function over square matrices which zeros out the upper triangular part of the input including the diagonal. Then the extended formulation space $\Xcal$ is
\begin{displaymath}
A \xvec \leq \bvec, \qquad \xvec \geq \zero
\end{displaymath}
or equivalently with slack variables $\svec$
\begin{displaymath}
A \xvec + \svec = \bvec, \qquad \xvec, \svec \geq \zero
\end{displaymath}
\end{lemma}
\begin{proof}
Let $\vvec^k$ be the vector in $\Vcal$ after going through the $k$th reflection relation. Also denote the $k$th column of $M$ by $M_k$. Observe that $\vvec^0 = \cvec$ and $\vvec^k = \cvec  + \sum_{i=1}^{k} M_i x_i$. Let $M_k = \evec_r - \evec_s$. Then, using (\ref{eq:single-reflection-formulation}), the inequality associated with the $k$th row of $A \xvec \leq \bvec$ will be obtained as below:
\begin{align*}
&x_k 
\leq v^{k-1}_s - v^{k-1}_r 
= - M_k^T \vvec^{k-1} 
= - M_k^T \rbr{\cvec  + \sum_{i=1}^{k-1} M_i x_i} \\
&\longrightarrow x_k + \sum_{i=1}^{k-1} M_k^T M_i x_i \leq -M_k^T \cvec = b_k
\end{align*}
Thus:
\begin{displaymath}
\forall \, i,j \in [m] \quad A_{ij} = \begin{cases}
M_i^T M_j & i>j \\
1 & i=j \\
0 & i<j
\end{cases}, \qquad
\forall \, k \in [m] \quad b_k = -M^T_k \cvec
\end{displaymath}
which concludes the proof.

\end{proof}

\section{Projection onto Each Constraint}
\label{sec:proj-to-constraint}
Each constraint of the polytope in the augmented formulation is of the form $\avec^T \wvec = a_0$. 
Formally, the projection $\wvec^*$ of a give point $\wvec$ to this constraint is solution to the following:
\begin{displaymath}
\underset{\avec^T \wvec^* = a_0}{\arg\min} \sum_{i} \, w^*_i \log \rbr{\frac{w^*_i}{w_i}} + w_i - w^*_i
\end{displaymath}

Finding the solution to the projection above for general hyperplanes and Bregman divergence can be found in Section 3 of \cite{dhillon2007matrix}. Nevertheless, for the sake of completeness, we also provide the solution for the particular case of Huffman trees and permutations described by $\Wcal$ in Section \ref{sec:instantiations} as well. Using the method of Lagrange multipliers, we have:
\begin{align*}
&L(\wvec^*, \mu) = \sum_{i} \, w^*_i \log \rbr{\frac{w^*_i}{w_i}} + w_i - w^*_i 
- \mu \rbr{\sum_{j=1}^{2m+n} a_i \, w^*_i - a_0 } \\
&\frac{\partial L}{\partial w^*_i} = \log \rbr{\frac{w^*_i}{w_i}} - \mu a_i = 0, 
\quad \forall i \in [n+2m] \\
&\frac{\partial L}{\partial \mu} = \sum_{j=1}^{2m+n} a_i \, w^*_i - a_0 = 0
\end{align*}

Replacing $\rho = e^{- \mu}$, we have $w^*_i = w_i \, \rho^{a_i}$. By enforcing $\frac{\partial L}{\partial \mu} =0$, one needs to find $\rho > 0$ such that:
\begin{equation}
\label{eq:single-proj}
\sum_{i=1}^{n+2m} a_i \, w_i \, \rho^{a_i} - a_0 = 0
\end{equation}

Observe that due to the structure of matrices $M$ and $A$ (see Lemma \ref{lemma:A-matrix}), $a_i \in \ZZ$ and $a_i \geq -1$ for all $i \in [n+2m]$, and furthermore $a_0 \geq 0$. Thus we can re-write the equation (\ref{eq:single-proj}) as the polynomial below:
\begin{displaymath}
f(\rho) = \phi_k \rho^k + \ldots + \phi_2 \rho^2 - \phi_1 \rho - \phi_0 = 0
\end{displaymath}
in which all $\phi_i$'s are positive real numbers. Note that $f(0) < 0$ and $f(\rho) \rightarrow + \infty$ as $\rho \rightarrow + \infty$. Thus $f(\rho)$ has at least one positive root. However, it can not have more than one positive roots and we can prove it by contradiction. Assume that there exist $0<r_1<r_2$ such that $f(r_1)=f(r_2)=0$. Since $f$ is convex on positive real line, using Jensen's inequality, we can obtain the contradiction below:
\begin{displaymath}
0 = f(r_1) = f \rbr{ \frac{r_2 - r_1}{r_2} \times 0 + \frac{r_1}{r_2} \times r_2 }
< \frac{r_2 - r_1}{r_2} f(0) + \frac{r_1}{r_2} f( r_2) = \frac{r_2 - r_1}{r_2} f(0) < 0
\end{displaymath}

Therefore $f$ has exactly one positive root which can be found by Newton's method starting from a sufficiently large initial point. Note that if the constraint belongs to $\vvec = M \xvec + \cvec$, since all the coefficinets are in $\{-1, 0, 1\}$, the $f$ will be quadratic and the positve root can be found through the closed form formula.

\section{Proof of Theorem \ref{thm:fast-pred}}
\label{sec:proof_thm_fast-pred}

\begin{proof}
Let $\xvec = [x_1, x_2, \ldots, x_m]^T$. Using induction, we prove that by the end of the $i$th loop of Algorithm \ref{alg:prediction}, the obtained distribution $\Dcal^{(i)}$ has the right expectation for $\xvec^{(i)} = [x_1 \ldots x_i \, 0 \ldots 0]$. 
Concretely, $\sum_{\hvec \in \Hcal} P_{\Dcal^{(i)}}[\hvec] \cdot \hvec = M \, \xvec^{(i)} + \cvec$. 
The desired result is obtained by setting $i=m$ as $\vvec = M \, \xvec + \cvec$ (see Appendix \ref{sec:struc-ext-form}).  
The base case $i=0$ (i.e.~before the first loop of the algorithm) is indeed true, since $\Dcal^{(0)}$ is initialized to follow $P_{\Dcal^{(0)}}[\cvec] = 1$, and $\xvec^{(0)} = \zero$, thus we have $\vvec^{(0)} = M \xvec^{(0)} + \cvec = \cvec$. 
Now assume that by the end of the $(k-1)$st iteration we have the right distribution $\Dcal^{(k-1)}$, namely $\vvec^{(k-1)} = \sum_{\hvec \in \Hcal} P_{\Dcal^{(k-1)}}[\hvec] \cdot \hvec$. 
Also assume that the $k$th comparator is applied on $i$th and $j$th element\footnote{Note that $j>i$ as in sorting networks the swap value is propagated to lower wires}. Thus the $k$th column of $M$ will be $M_k = \evec_i - \evec_j$. Now, according to (\ref{eq:single-reflection-formulation}) the swap capacity at $k$th comparator is:
\begin{align}
x_k + s_k &= v^{k-1}_j - v^{k-1}_i  \nonumber \\
&= \sum_{\hvec \in \Hcal} P_{\Dcal^{(k-1)}}[\hvec] \cdot  (h_{j} - h_{i} ) \nonumber \\
&= - \sum_{\hvec \in \Hcal} P_{\Dcal^{(k-1)}}[\hvec] \cdot  M_k^T \, \hvec \nonumber \\
&= - \, M_k^T \sum_{\hvec \in \Hcal} P_{\Dcal^{(k-1)}}[\hvec] \cdot  \hvec
\label{eq:swap_capacity}
\end{align}

Now observe:
\begin{align*}
\vvec^{(k)} 
&= M \xvec^{(k)} + \cvec  \\
&= x_k M_k  + M \xvec^{(k-1)} + \cvec \\
&= x_k M_k  + \vvec^{(k-1)} \\
&= x_k M_k  + \sum_{\hvec \in \Hcal} P_{\Dcal^{(k-1)}}[\hvec] \cdot \hvec \\
&= \frac{x_k}{x_k+s_k} \, M_k \rbr{x_k+s_k}   + \sum_{\hvec \in \Hcal} P_{\Dcal^{(k-1)}}[\hvec] \cdot \hvec \\
&= - \, \frac{x_k}{x_k+s_k} \ M_k M_k^T \sum_{\hvec \in \Hcal} P_{\Dcal^{(k-1)}}[\hvec] \cdot \hvec   
+ \sum_{\hvec \in \Hcal} P_{\Dcal^{(k-1)}}[\hvec] \cdot \hvec && \text{According to (\ref{eq:swap_capacity})}\\
&= \rbr{ I - \, \frac{x_k}{x_k+s_k} \, M_k M_k^T} \sum_{\hvec \in \Hcal} P_{\Dcal^{(k-1)}}[\hvec] \cdot \hvec   \\
&= \rbr{ \frac{s_k}{x_k+s_k} I + \frac{x_k}{x_k+s_k} \, T_{ij} } \sum_{\hvec \in \Hcal} P_{\Dcal^{(k-1)}}[\hvec] \cdot \hvec    \\
&= \sum_{\hvec \in \Hcal} 
\underbrace{\frac{s_k}{x_k+s_k} \, P_{\Dcal^{(k-1)}}[\hvec]}_{P_{\Dcal^{(k)}}[\hvec]}
 \cdot \hvec 
+ \underbrace{\frac{x_k}{x_k+s_k} \, P_{\Dcal^{(k-1)}}[\hvec]}_{P_{\Dcal^{(k)}}[T_{ij} \, \hvec]}
  \cdot T_{ij} \, \hvec  \\ 
&= \sum_{\hvec \in \Hcal} P_{\Dcal^{(k)}}[\hvec] \cdot \hvec
\end{align*}
in which $T_{ij}$ is a row-switching matrix that is obtained form switching $i$th and $j$th row from identity matrix. 
For Huffman trees, the linear maps introduced in \cite{pashkovich2012extended} are used to set the depths of the leaves.
It is straightforward to see that these linear maps maintain
the equality $\vvec^{(k)} = \sum_{\hvec \in \Hcal} P_{\Dcal^{(k)}}[\hvec] \cdot \hvec$ when applied to $\vvec^{(k)}$
and all $\hvec$'s in $\Hcal$. 
This concludes the inductive proof. \\

The final distribution $\Dcal$ over objects $\hvec \in \Hcal$ is decomposed into individual actions of swap/pass through the network of comparators independently. Thus one can draw an instance according to the distribution by simply doing independent Bernoulli trials associated with the comparators. It is also easy to see that the time complexity of the algorithm is $O(m)$ since one just needs to do $m$ Bernoulli trials.
\end{proof}

\section{Additional Loss with Approximate Projection}
\label{sec:apprx-breg}
Each iteration of Bregman Projection is described in Appendix \ref{sec:proj-to-constraint}. Since it is basically finding a positive root of a polynomial (which $n/(n+m)$ of the time is quadratic), each iteration is arguably efficient. Now suppose, using iterative Bregman projections, we reached at $\wvechat = (\vvechat, \xvechat, \svechat)$ which is $\epsilon$-close to the exact projection $\wvec = (\vvec, \xvec, \svec)$, that is $\| \wvec - \wvechat \|_2 < \epsilon$. In this analysis, we work with a two-level approximation: 1) approximating mean vector $\vvec$ by the mean vector $\vvectil := M \xvechat + \cvec$ and 2) approximating the mean vector $\vvectil$ by the mean vector $\vvec(\pvechat)$ (where $\pvechat = \xvechat / (\xvechat + \svechat)$ with coordinate-wise division) obtained from Algorithm \ref{alg:prediction} with $\xvechat$ and $\svechat$ as input. First, observe that:
\begin{align}
\| \vvec - \vvectil \|_2 &= \| M \, (\xvec - \xvechat) \|_2 \nonumber \\
&\leq \| M \|_F \, \| \xvec - \xvechat \|_2 \nonumber \\
&\leq (\sqrt{2 \, n}) \, \epsilon
\label{eq:close_v_vtil}
\end{align}

Now suppose we run Algorithm \ref{alg:prediction} with $\xvechat$ and $\svechat$ as input. Similar to Appendix \ref{sec:proof_thm_fast-pred}, let $M_k$ be the $k$-th column of $M$, and let $T_{\alpha\beta}$ be the row-switching matrix that is obtained from switching $\alpha$-th and $\beta$-th row in identity matrix. 
%
%
Additionally, let $\vvec^{(k)}(\pvechat)$ be the mean vector associated with the distribution $\Dcal^{(k)}$ obtained by the end of $k$-th loop of the Algorithm \ref{alg:prediction} i.e.~$\vvec^{(k)}(\pvechat) := \sum_{\hvec \in \Hcal} P_{\Dcal^{(k)}}[\hvec] \cdot \hvec$ (so $\vvec^{(m)}(\pvechat)=\vvec(\pvechat)$). Also for all $k \in \{1..m\}$ define $\vvectil^{(k)} := \cvec+ \sum_{i=1}^k  M_i \, \xhat_i $ (thus $\vvectil^{(m)} = \vvectil $). Furthermore, let $\deltavec^{(k)} := \vvec^{(k)}(\pvechat) - \vvectil^{(k)}$. Now we can write:

\begin{align*}
\vvec^{(k)}(\pvechat) 
&= \sum_{\hvec \in \Hcal} P_{\Dcal^{(k)}}[\hvec] \cdot \hvec \\
&= \sum_{\hvec \in \Hcal} \frac{\shat_k}{\xhat_k + \shat_k} P_{\Dcal^{(k-1)}}[\hvec] \cdot \hvec
+ \frac{\xhat_k}{\xhat_k + \shat_k} P_{\Dcal^{(k-1)}}[\hvec] \cdot T_{\alpha\beta} \,\hvec \\
&= (\frac{\shat_k}{\xhat_k + \shat_k} \, I + \frac{\xhat_k}{\xhat_k + \shat_k}\, T_{\alpha\beta}) \, 
\sum_{\hvec \in \Hcal} P_{\Dcal^{(k-1)}}[\hvec] \cdot \hvec \\
&= ( I - \frac{\xhat_k}{\xhat_k + \shat_k}\, M_k \, M_k^T) \, \vvec^{(k-1)}(\pvechat)  \qquad \text{ since } I-T_{\alpha\beta} = M_k \, M_k^T \\
&= ( I - \frac{\xhat_k}{\xhat_k + \shat_k}\, M_k \, M_k^T) \, \vvectil^{(k-1)}
+ ( I - \frac{\xhat_k}{\xhat_k + \shat_k}\, M_k \, M_k^T) \, \deltavec^{(k-1)}  \\
&= ( I - \frac{\xhat_k}{\xhat_k + \shat_k}\, M_k \, M_k^T) \, \left( \cvec + \sum_{i=1}^{k-1} M_i \, \xhat_i \right) 
+ ( I - \frac{\xhat_k}{\xhat_k + \shat_k}\, M_k \, M_k^T) \, \deltavec^{(k-1)}  \\
&=  \left( \cvec + \sum_{i=1}^{k-1} M_i \, \xhat_i  \right)
 - \frac{\xhat_k}{\xhat_k + \shat_k}\, M_k \, M_k^T \, \left(\cvec+ \sum_{i=1}^{k-1} M_i \, \xhat_i \right) \\
&\qquad+ ( I - \frac{\xhat_k}{\xhat_k + \shat_k}\, \, M_k \, M_k^T) \, \deltavec^{(k-1)}  \\
&=  \vvectil^{(k)} 
\underbrace{- M_k \xhat_k
 - \frac{\xhat_k}{\xhat_k + \shat_k}\, M_k \, M_k^T \, \left(\cvec+ \sum_{i=1}^{k-1} M_i \, \xhat_i \right) 
+ ( I - \frac{\xhat_k}{\xhat_k + \shat_k}\, M_k \, M_k^T) \, \deltavec^{(k-1)}}_{\deltavec^{(k)}}
\end{align*}

Now define $\phat_k := \frac{\xhat_k}{\xhat_k + \shat_k} $ and let $\text{err}_k := - \, M_k^T \, \left(\cvec+ \sum_{i=1}^{k-1} M_i \, \xhat_i  \right)  - (\xhat_k + \shat_k)$, which is -- according to Lemma \ref{lemma:A-matrix} -- the error in the $k$-th row of $A \xvec + \svec = \bvec $ using $\xvechat$ and $\svechat$ i.e.~amount by which $(A \xvechat + \svechat)_k$ falls short of $b_k$, violating the $k$-th constraint of $A \xvec + \svec = b$. Thus we obtain:
\begin{align*}
\deltavec^{(k)}  
&= - M_k \xhat_k  + \phat_k \, M_k \, (\xhat_k + \shat_k + \text{err}_k) + ( I - \phat_k \, M_k \, M_k^T) \, \deltavec^{(k-1)}  \\
&= \phat_k \, M_k \, \text{err}_k + ( I - \phat_k \, M_k \, M_k^T) \, \deltavec^{(k-1)}  && \phat_k = \frac{\xhat_k}{\xhat_k + \shat_k} \\
\end{align*}

Observe that $\deltavec^{(0)} =  \cvec - \cvec = \zero$. Thus, by unrolling the recurrence relation above, we have:
\begin{align*}
\vvec(\pvechat) - \vvectil = \deltavec^{(m)} = \sum_{k=1}^m \phat_k \, M_k \, \text{err}_k \, \prod_{i=k+1}^{m} ( I - \phat_i \, M_i \, M_i^T)
\end{align*}

Note that since $I - \phat_i \, M_i \, M_i^T$ is a $n \times n$ doubly-stochastic matrix, $\prod_{i=1}^{k-1} ( I - \phat_i \, M_i \, M_i^T)$ is also a $n \times n$ doubly-stochastic matrix, and consequently, its Frobenius norm is at most $\sqrt{n}$. Thus we have:
\begin{align}
\| \vvec(\pvechat) - \vvectil \|_2 = \| \deltavec^{(m)} \|_2 
& \leq \sum_{k=1}^m | \phat_k |  \,  \| M_k \|_2 \, | \text{err}_k |  \, \sqrt{n} 
\leq \sqrt{2 \, n} \, \sum_{k=1}^m  | \text{err}_k | \nonumber \\
& = \sqrt{2 \, n} \, \| \text{\textbf{err}} \|_1 && \textbf{err}:=(\text{err}_1,\ldots,\text{err}_m) \nonumber \\
& \leq \sqrt{2 \, n \, m} \, \| \text{\textbf{err}} \|_2 
\label{eq:close_vp_vtil}
\end{align}

Observe that we can bound the 2-norm of the vector \textbf{err} as follows:
\begin{align}
\| \text{\textbf{err}} \|_2 
&= \| -A \xvechat - \svechat + \bvec \|_2 \nonumber \\
&= \| A (\xvec - \xvechat ) +(\svec - \svechat) \|_2 && \bvec = A \xvec + \svec \nonumber \\
&\leq \| A \|_F \, \| \xvec - \xvechat \|_2 + \| \svec - \svechat \|_2 \nonumber \\
&\leq \| M^T M + I \|_F \, \epsilon + \epsilon \nonumber \\
&\leq \left( \| M^T \|_2 \, \| M \|_2 + \| I \|_2 \right) \, \epsilon + \epsilon \nonumber \\
&= ( 2 n + \sqrt{n} + 1) \, \epsilon 
\label{eq:err_bound}
\end{align}

Therefore, if we perform Algorithm \ref{alg:prediction} with inputs $\xvechat$ and $\svechat$, combining the inequalities (\ref{eq:close_v_vtil}), (\ref{eq:close_vp_vtil}), and (\ref{eq:err_bound}), the generated mean vector $\vvec(\pvechat)$ can be shown to be close to the mean vector $\vvec$ associated with the exact projection:
\begin{align*}
\| \vvec - \vvec(\pvechat) \|_2 
&\leq \| \vvec - \vvectil \|_2  + \| \vvectil - \vvec(\pvechat) \|_2  \\
&\leq (\sqrt{2 \, n}) \, \epsilon  + \sqrt{2 \, n \, m} \, ( 2 n + \sqrt{n} + 1) \, \epsilon  \\
&= \sqrt{2 \, n} \, ( 1  + \sqrt{ m} \, ( 2 n + \sqrt{n} + 1) ) \, \epsilon 
\end{align*}

Now we can compute the total expected loss using approximate projection:
\begin{align*}
\left| \sum_{t=1}^T \vvec^{t-1}(\pvechat) \cdot \vec{\ell}^t \right| 
&= \left| \sum_{t=1}^T \left( \vvec^{t-1} + (\vvec^{t-1} -  \vvec^{t-1}(\pvechat)) \right) \cdot \vec{\ell}^t \right| \\
&= \left| \sum_{t=1}^T \vvec^{t-1} \cdot \vec{\ell}^t + \sum_{t=1}^T (\vvec^{t-1} -  \vvec^{t-1}(\pvechat)) \cdot \vec{\ell}^t \right| \\
&\leq \left| \sum_{t=1}^T \vvec^{t-1} \cdot \vec{\ell}^t \right| + \left| \sum_{t=1}^T (\vvec^{t-1} -  \vvec^{t-1}(\pvechat)) \cdot \vec{\ell}^t \right| \\
&\leq \left| \sum_{t=1}^T \vvec^{t-1} \cdot \vec{\ell}^t \right| +  \sum_{t=1}^T  \| \vvec^{t-1} -  \vvec^{t-1}(\pvechat) \|_2 \, \| \vec{\ell}^t \|_2\\
&\leq \left| \sum_{t=1}^T \vvec^{t-1} \cdot \vec{\ell}^t \right| 
+  T \,  \left( \sqrt{2 \, n} \, ( 1  + \sqrt{ m} \, ( 2 n + \sqrt{n} + 1) ) \, \epsilon  \right) \, \sqrt{n}
\end{align*}

For Huffman trees, the linear maps introduced in \cite{pashkovich2012extended} have this property that 
$\| F(\avec) - F(\avec') \|_2 \leq \| \avec - \avec' \|_2 $ for all vectors $\avec$ and $\avec'$ where $F(\cdot)$ is the linear map.
Using this property, it is straightforward to observe that this analysis can be extended for Huffman trees in which 
these linear maps are also used along with the reflection relations.

Setting $\epsilon = \frac{1}{(\sqrt{2} \, n) \, ( 1  + \sqrt{ m} \, ( 2 n + \sqrt{n} + 1) ) \, T}$, we add at most one unit to the expected cumulative loss with exact projections.\\

\end{document}